\documentclass[accepted]{uai2023} 

\usepackage[american]{babel}

\usepackage{natbib} 
\usepackage{mathtools} 
\usepackage{booktabs} 
\usepackage{tikz} 

\usepackage{times}
\usepackage{soul}
\usepackage{url}
\usepackage[utf8]{inputenc}
\usepackage{algorithm}
\usepackage[switch]{lineno}

\usepackage{microtype}
\usepackage{xcolor}
\usepackage{graphicx}
\usepackage{booktabs}
\usepackage{hyperref}
\usepackage{natbib}
\usepackage{amsmath}
\usepackage{amsthm}
\usepackage{amssymb}
\usepackage{mathtools}
\usepackage{setspace}
\usepackage{adjustbox}
\usepackage{relsize}
\usepackage[capitalize,noabbrev]{cleveref}
\usepackage{multirow}
\usepackage{multicol}
\usepackage{float}
\usepackage[titletoc,title]{appendix}
\usepackage{acronym}
\usepackage{roboto}
\usepackage{nccmath}
\usepackage{psfrag}
\usepackage[hyphenbreaks]{breakurl}
\usepackage{caption}
\usepackage{subcaption}
\usepackage{dsfont}
\usepackage{algpseudocode}
\usepackage{latexsym,newlfont}
\newtheorem{theorem}{Theorem}

\newcommand\V[1]  { \mathbf{#1} }

\newcommand\B[1]  { \boldsymbol{#1} }

\newcommand\set[1] {\mathcal{#1}}

\newcommand\up[1] {\mathrm{#1}}

\acrodef{RFF}[RFF]{random Fourier features}
\acrodef{RRM}[RRM]{robust risk minimization}
\acrodef{MRC}[MRC]{minimax risk classifier}
\acrodef{RFE}[RFE]{recursive feature elimination}
\acrodef{SVM}[SVM]{support vector machine}
\acrodef{LP}[LP]{linear program}
\acrodef{LR}[LR]{logistic regression}
\acrodef{SOA}[SOA]{state-of-the-art}
\acrodef{DT}[DT]{decision tree}
\acrodef{MRMR}[MRMR]{minimum redundancy maximum relevancy}
\acrodef{MI}[MI]{mutual information}
\acrodef{ANOVA}[ANOVA]{analysis of variance}

\title{Efficient Learning of Minimax Risk Classifiers in High Dimensions}

\author[1]{\href{mailto:<kbondugula@bcamath.org>?Subject=Your UAI 2023 paper}{Kartheek Bondugula}{}}
\author[1,2]{\href{mailto:<smazuelas@bcamath.org>?Subject=Your UAI 2023 paper}{Santiago Mazuelas}{}}
\author[1]{\href{mailto:<aperez@bcamath.org>?Subject=Your UAI 2023 paper}{Aritz P\'{e}rez}{}}
\affil[1]{%
    Basque Center for Applied Mathematics (BCAM)\\
    Bilbao, Spain
}
\affil[2]{%
    IKERBASQUE-Basque Foundation for Science
}
  
\begin{document}
\maketitle

\begin{abstract}
\mbox{High-dimensional} data is common in multiple areas, such as health care and genomics, where the number of features can be tens of thousands. In such scenarios, the large number of features often leads to inefficient learning. Constraint generation methods have recently enabled efficient learning of L1-regularized \acp{SVM}. In this paper, we leverage such methods to obtain an efficient learning algorithm for the recently proposed \acp{MRC}. The proposed iterative algorithm also provides a sequence of worst-case error probabilities and performs feature selection. Experiments on multiple \mbox{high-dimensional} datasets show that the proposed algorithm is efficient in \mbox{high-dimensional} scenarios. In addition, the worst-case error probability provides useful information about the classifier performance, and the features selected by the algorithm are competitive with the state-of-the-art.
\end{abstract}

\section{Introduction}\label{sec:intro}
\mbox{High-dimensional data} is common in multiple areas such as health care and genomics. A typical example of high-dimensional supervised classification problem in genomics is to separate healthy patients from cancer patients based on gene expression data with tens of thousands of features \citep{GuyIsaEtal:02}. In addition to \mbox{high-dimensional} raw data, learning methods often perform a \mbox{high-dimensional} representation of the input data vector in order to improve the classification performance \citep{SonRatSch:06, RahRec:08, LiuYon:17}. Learning in such \mbox{high-dimensional} settings often leads to highly complex optimization processes because the number of variables involved in the optimization increases with the number of features \citep{ShiYin:10, YuaGuo:12}. 

In addition to \mbox{high-dimensional} data, a limited number of samples is common in the above-mentioned applications e.g., tens of thousands of features but less than 100 patients \citep{GuyEli:03, BroGavEtal:12}. In such scenarios, the conventional performance assessment based on cross-validation can be unreliable \citep{Var:18, vabalas2019}. In addition, these cross-validation estimates also increase the computational cost as they require learning multiple classifiers.

Multiple methods have been proposed to improve the learning efficiency in high dimensions (see e.g., \cite{YuaGuo:12}). These methods are based on several approaches such as coordinate descent \citep{HsiCho:08, YuHsi:11}, interior-point method \citep{koh2007interior}, and stochastic subgradient \citep{shalev2007pegasos}. In addition, multiple classification and regression methods exploit the parameters’ sparsity induced by convex penalties and regularization terms in high-dimensional settings \citep{MigKrzLu:20, CelMon:22}. Recently, constraint generation techniques have enabled the efficient learning of L1-regularized \acfp{SVM} for cases with a large number of features \citep{DedAntEtal:22}. These techniques obtain improved efficiency for the \ac{LP} corresponding with binary L1-\acp{SVM} due to the sparsity in the solution. 

In this paper, we present a learning algorithm for the recently proposed \acfp{MRC} \citep{MazZanPer:20, MazRomGrun:23}. The presented algorithm provides efficient learning in high dimensions, obtains worst-case error probabilities that can serve to assess performance, and performs feature selection. Specifically, the main contributions in the paper are as follows.
\vspace{-0.24cm}
\begin{itemize}[itemsep=0pt]
\item We present a learning algorithm for \acp{MRC} that utilizes constraint generation methods to significantly improve the efficiency in high dimensions.
\item The presented algorithm utilizes a greedy feature selection approach that achieves a fast decrease in worst-case error probability while using a small number of features.
\item Our theoretical results show that the proposed algorithm obtains a sequence of classifiers with decreasing worst-case error probabilities that converges to the solution obtained using all the features.
\item Using multiple benchmark datasets, we experimentally show that the proposed algorithm performs efficient learning in high-dimensional settings. In addition, the algorithm obtains an error assessment for the classifier and performs effective feature selection.
\end{itemize}

\paragraph{Notations:} For a set $\set{S}$, we denote its complement as $\set{S}^{\up{c}}$ and its cardinality as $|\set{S}|$; bold lowercase and uppercase letters represent vectors and matrices, respectively; vectors subscripted by a set of indices represent subvectors obtained by the components corresponding to the indices in the set; matrices subscripted by a set of indices represent submatrices obtained by the columns corresponding to the indices in the set; $\V{I}$ denotes the identity matrix; $\mathds{1} \{\cdot\}$ denotes the indicator function; $\B{1}$ denotes a vector of ones; for a vector $\V{v}$, $|\V{v}|$ and $(\V{v})_+$, denote its component-wise absolute value and positive part, respectively; $\|.\|_1$ denotes the 1-norm of its argument; $\otimes$ denotes the Kronecker product; $\preceq$ and $\succeq$ denote vector inequalities; $\mathbb{E}_{\up{p}}\{\,\cdot\,\}$ denotes the expectation of its argument with respect to distribution $\up{p}$; and $\V{e}_{i}$ denotes the $i$-th vector in a standard basis.

\section{Preliminaries}
In this section, we describe the setting addressed in the paper and MRC methods that minimize the worst-case error probability.

\subsection{Problem formulation}
Supervised classification uses instance-label pairs to determine classification rules that assign labels to instances. We denote by $\mathcal{X}$ and $\mathcal{Y}$ the sets of instances and labels, respectively, with $\mathcal{Y}$ represented by $\{1, 2, \ldots, |\mathcal{Y}|\}$. We denote by $\text{T}(\mathcal{X}, \mathcal{Y})$ the set of all classification rules (both randomized and deterministic) and we denote by $\up{h}(y|x)$ the probability with which rule $\up{h} \in \text{T}(\mathcal{X}, \mathcal{Y})$ assigns label $y \in \mathcal{Y}$ to instance $x \in \mathcal{X}$ ($\up{h}(y|x) \in \{0, 1\}$ for deterministic classification rules). In addition, we denote by $\Delta(\mathcal{X}\times\mathcal{Y})$ the set of probability distributions on $\mathcal{X} \times \mathcal{Y}$ and by $\ell(\up{h}, \up{p})$ the expected 0-1 loss of  the classification rule $\up{h}\in \text{T}(\mathcal{X}, \mathcal{Y})$ with respect to distribution $\up{p} \in \Delta(\mathcal{X}\times\mathcal{Y})$. If $\up{p}^*\in \Delta(\set{X}\times\set{Y})$ is the underlying distribution of the instance-label pairs, then $\ell(\up{h},\up{p}^*)$ is the classification error probability or classification risk of rule $\up{h}$ denoted in the following as $\mathcal{R}(\up{h})$, that is, $\mathcal{R}(\up{h}) \coloneqq \ell(\up{h},\up{p}^*)$.
 
Instance-label pairs can be represented as real vectors by using a feature mapping $\Phi: \set{X}\times \set{Y} \rightarrow \mathbb{R}^m$.  The most common way to define such feature mapping is using multiple features over instances together with one-hot encodings of labels as follows (see e.g., \cite{MehRos:18})
\begin{align}
\label{eq:feature}
 \Phi(x, y)  = & \, \V{e}_{y} \otimes {\Psi}(x)=  \big{[}\mathds{1}\{y = 1\} {\Psi}(x)^{\text{T}}, \\
 & \mathds{1}\{y = 2\} {\Psi}(x)^{\text{T}}, \ldots, \mathds{1}\{y = |\mathcal{Y}|\} {\Psi}(x)^{\text{T}}\big{]}^{\text{T}}\nonumber
\end{align}
 where the map $\Psi : \mathcal{X} \rightarrow \mathbb{R}^d$ represents instances as real vectors of size $d$. This map can be just the identity $\Psi(x)=x$ or given by a feature representation such as that provided by \ac{RFF} \citep{RahRec:08}, that is
\begin{align}
\label{eq:RRF}
\Psi(x) = \big{[}&\cos(\V{u}_1^{\text{T}} \V{x}), \cos(\V{u}_2^{\text{T}} \V{x}),..., \cos(\V{u}_D^{\text{T}} \V{x}), \\
&\sin(\V{u}_1^{\text{T}} \V{x}), \, \sin(\V{u}_2^{\text{T}} \V{x}), ..., \, \sin(\V{u}_D^{\text{T}} \V{x})\big{]}^{\text{T}} \nonumber
\end{align}
for $D$ random Gaussian vectors $\V{u}_1, \V{u}_2, ..., \V{u}_D \sim N(0, \gamma \V{I})$ with covariance given by the scaling parameter $\gamma$ (see e.g., \cite{RahRec:08}). 

In this paper we consider scenarios in which the features' dimensionality ($d$) is large, and propose efficient learning methods for the recently presented \acp{MRC}, which are briefly described in the following.

\subsection{Minimax risk classifiers}

\acp{MRC} are classification rules that minimize the worst-case error probability over distributions in an uncertainty set \citep{MazZanPer:20, MazSheYua:22, MazRomGrun:23}. Specifically, such rules are solutions to the minimax risk problem defined as
 \begin{equation}
 \label{eq:minmaxrisk}
\underset{\up{h} \in \text{T}(\mathcal{X}, \mathcal{Y})}{\min} \, \underset{\up{p} \in \mathcal{U}}{\max} \; \ell(\up{h}, \up{p})
\end{equation}
where $\set{U}$ is an uncertainty set of distributions determined by expectation estimates as
\begin{equation}
\label{eq:us}
\mathcal{U} = \{\up{p} \in \Delta (\mathcal{X} \times \mathcal{Y}) : \left|\mathbb{E}_{\up{p}}\{\Phi(x, y)\}  - \B{\tau} \right| \preceq \B{\lambda}\}
 \end{equation}
where $\B{\tau}$ denotes the mean vector of expectation estimates corresponding with the feature mapping $\Phi$, and $\B{\lambda} \succeq \V{0}$ is a confidence vector that accounts for inaccuracies in the estimate. The mean and confidence vectors can be obtained from the training samples $\{(x_i,y_i)\}_{i=1}^n$ as
\begin{align}
\B{\tau}=\frac{1}{n}\sum_{i=1}^{n}\Phi(x_i,y_i), \  \B{\lambda}= \frac{\V{s}}{\sqrt{n}}
\end{align}
where $\V{s}$ denotes the vector formed by the component-wise sample standard deviations of $\{\Phi(x_i,y_i)\}_{i=1}^n$.

As described in  \cite{MazZanPer:20, MazSheYua:22, MazRomGrun:23} using the expected 0-1 loss $\ell$, the MRC rule $\up{h}$ solution of~\eqref{eq:minmaxrisk} is given by a linear combination of the feature mapping with coefficients given by a vector $\B{\mu}^* \in \mathbb{R}^{m}$. Specifically, the \ac{MRC} rule $\up{h}$ given by $\B{\mu}^{*}$ assigns label $\hat{y} \in \mathcal{Y}$ to instance $x \in \mathcal{X}$ with probability
\begin{equation}
\label{eq:prob}
\up{h}(\hat{y}|x) = \left\{\begin{matrix}\left(\Phi(x, \hat{y})^{\text{T}} \B{\mu}^*  -  \varphi(\B{\mu}^*) \right)_+ /d_{x} & \text{if} \, d_{x} \neq 0\\
1/|\mathcal{Y}| & \text{if} \, d_{x} = 0 \end{matrix}\right.
\end{equation}
\begin{align}
\nonumber
\label{eq:varphi}
\hspace{-0.2cm}\mbox{where}\hspace{0.3cm}\varphi(\B{\mu}) &= {\underset{x \in \mathcal{X}, \set{C} \subseteq \mathcal{Y}}{\max}} \Big{(}\frac{\sum_{y \in \set{C}}\Phi(x, y)^{\text{T}}\B{\mu} - 1}{|\set{C}|}\Big{)}\hspace{0.2cm}\\
d_{x}&=\sum_{y \in \mathcal{Y}} \left(\Phi(x, y)^{\text{T}} \B{\mu}^{*} - \varphi(\B{\mu}^{*}) \right)_+.\nonumber
\end{align}
Note that the corresponding deterministic classification rule $\up{h}^{\up{d}}$ assigns a label that maximizes the probability in \eqref{eq:prob} and is given by \mbox{${\arg} \max_{y \in \set{Y}} \up{h}(y|x) = \arg \max_{y \in \set{Y}} {\Phi}(x,y)^{\text{T}} \B{\mu}^*$}. In addition, the classification risk of such classification rule $\up{h}^{\up{d}}$ is bounded by twice the error probability of the \ac{MRC} rule $\up{h}$, that is,
\begin{equation}
\label{eq:cl_risk}
\mathcal{R}({\up{h}^{\up{d}}}) \leq 2\mathcal{R}(\up{h})
\end{equation}
since $1-{\up{h}^{\up{d}}}(y|x) \leq 2(1-\up{h}(y|x))$ for any $x \in \mathcal{X}$, $y \in \mathcal{Y}$.

The vector $\B{\mu}^*$ that determines the \mbox{0-1} \ac{MRC} rule corresponding to \eqref{eq:minmaxrisk} is obtained by solving the convex optimization problem \citep{MazZanPer:20, MazRomGrun:23}
\begin{equation}
\label{eq:mrc}
\up{R}^{*} = \underset{\B{\mu}}{\min} \; 1 - \B{\tau}^{\text{T}} \B{\mu} + \varphi(\B{\mu}) + \B{\lambda}^{\text{T}} | \B{\mu} |.
\end{equation}
\noindent
The minimum value $\up{R}^{*}$ of \eqref{eq:mrc} equals the minimax risk value of \eqref{eq:minmaxrisk}. Hence, it is the \ac{MRC}'s worst-case error probability for distributions included in the uncertainty set $\mathcal{U}$, and is an upper bound on the \ac{MRC}'s classification error when the underlying distribution is included in such uncertainty set.

The convex optimization problem \eqref{eq:mrc} of \acp{MRC} is \mbox{L1-penalized} and leads to sparsity in the coefficients $\B{\mu}^{*}$ corresponding to the features in $\Phi$. This sparsity in the coefficients implies that only a subset of features are sufficient to obtain the optimal worst-case error probability for \acp{MRC}. Therefore, in practice, efficient learning can be achieved by using the relevant subset of features in $\Phi$ to solve the \ac{MRC} optimization problem. This intuition leads to the efficient learning algorithm presented in this following.

\section{Efficient learning of MRCs in high dimensions}\label{sec:math}
This section describes the proposed learning algorithm based on constraint generation over an LP formulation of MRCs learning \eqref{eq:mrc}.

\subsection{LP formulation for \mbox{0-1} MRCs}

The theorem below presents an \ac{LP} formulation (primal and dual) for the \mbox{0-1} \ac{MRC} convex optimization problem \eqref{eq:mrc}.
\begin{theorem}\label{th:linear_form}
Let $\mathcal{S}$ be the set of pairs of instances and labels' subsets, i.e., $\mathcal{S}=\{(x,\set{C}):\  x \in \set{X}, \set{C} \subseteq \set{Y}, \set{C} \neq \emptyset\}$, and for each $i=1,2\ldots,|\mathcal{S}|$ corresponding to pair $(x,\mathcal{C})$ let
\begin{align}
\label{def:matrices}
\V{g}_{i} = \frac{\sum_{y \in \mathcal{C}}\Phi\left(x,y\right)}{|\mathcal{C}|}, \ 
 \up{b}_{i} = \frac{1}{\mathcal{|C|}}-1.
\end{align} 
Then, the \mbox{0-1} \ac{MRC} problem \eqref{eq:mrc} is equivalent to the \ac{LP}
\begin{align}
\arraycolsep=1.4pt
\label{eq:mrc_linear}
\begin{array}{ccc}
\mathcal{P}: & \underset{\substack{\B{\mu}_1, \B{\mu}_2, \nu}}{\min} & - (\B{\tau} - \B{\lambda})^{\up{T}}\B{\mu}_1 + (\B{\tau} + \B{\lambda})^{\up{T}}\B{\mu}_2 + \nu \\ 
 & \text{s.t.} & \B{\up{F}}(\B{\mu}_1 - \B{\mu}_2) - \nu\B{1} \preceq \B{\up{b}} \\
 & & \B{\mu}_1, \B{\mu}_2 \succeq 0
\end{array}
\end{align}
where the $i$-th row of matrix $\V{F} \in \mathbb{R}^{|\set{S}| \times m}$ is given by $\V{g}_i^{\up{T}}$, and the $i$-th component of vector $\V{b} \in \mathbb{R}^{|\mathcal{S}|}$ is given by $\up{b}_i$. In addition, the Lagrange dual of \eqref{eq:mrc_linear} is
\begin{align}
\label{eq:dual_mrc}
\def\arraystretch{1.2}
\begin{array}{ccc}
\mathcal{D}: & \underset{\B{\alpha}}{\max} & -\B{\up{b}}^{\up{T}}\B{\alpha} \\
& \text{s.t.} & \B{\tau} - \B{\lambda} \ \preceq \ \B{\up{F}}^{\up{T}}\B{\alpha} \ \preceq  \ \B{\tau} + \B{\lambda} \\
& & \B{1}^{\up{T}}\B{\alpha} = 1, \ \B{\alpha} \succeq 0
\end{array}
\end{align}
and if $\B{\mu}_{1}^{*}$, $\B{\mu}_{2}^{*}$, $\nu^{*}$ is a solution of \eqref{eq:mrc_linear}, we have that \mbox{$\B{\mu^{*}} = \B{\mu}_1^{*} - \B{\mu}_2^{*}$} is a solution of \eqref{eq:mrc}.
\end{theorem}
\begin{proof}
In the convex optimization problem \eqref{eq:mrc}, we introduce the additional variable $\nu \in \mathbb{R}$ given by $\nu = \varphi(\B{\mu}) + 1$. Then, using $\V{g}_{i}$ and $\B{\up{b}}$ as defined above, we have
\begin{align}
\begin{array}{cccc}
\nu & = & \underset{i \in \{1,2, \ldots, |\mathcal{S}|\}}{\max} & \V{g}_{i}^{\up{T}}\B{\mu} -\up{b}_{i}.
\end{array}
\end{align}
Therefore, the optimization problem \eqref{eq:mrc} becomes
\begin{align}\label{eq:mrc_intermediate}
\arraycolsep=1.4pt
\begin{array}{ccc}
\underset{\B{\mu}, \nu}{\min}  & - \B{\tau}^{\up{T}}\boldsymbol{\mu} + \B{\lambda}^{\up{T}}|\B{\mu}| + \nu \\
\text{s.t.} & \B{\up{F}}\B{\mu} - \nu\B{1} \preceq \B{\up{b}}.
\end{array}
\end{align}
\noindent
The optimization in \eqref{eq:mrc_intermediate} can be reformulated as the \ac{LP} in \eqref{eq:mrc_linear} using the change of variables $\B{\mu}_1 = (\B{\mu})_{+}$ and $\B{\mu}_2 = (-\B{\mu})_{+}$ which implies $\B{\mu} = \B{\mu}_1 -\B{ \mu}_2$ and $|\B{\mu}| = \B{\mu}_1 +\B{\mu}_2$. Finally, the dual in \eqref{eq:dual_mrc} is directly obtained from the \ac{LP} \eqref{eq:mrc_linear} (see e.g., Section 4.2 in \cite{DimTsi97}).
\end{proof}
The number of variables of the primal problem $\mathcal{P}$ in \eqref{eq:mrc_linear} (constraints in dual problem $\mathcal{D}$ in \eqref{eq:dual_mrc}) is given by the number of features of $\Phi$, $m=d|\mathcal{Y}|$. Therefore, the complexity of \ac{MRC} learning given by such \ac{LP} formulation is $O(m^3)$ which is not affordable in \mbox{high-dimensional} settings. The next section presents an efficient learning algorithm for \acp{MRC} in high dimensions based on constraint generation methods.

\subsection{Algorithm for efficient learning}
\ac{MRC} learning can be efficiently carried out in \mbox{high-dimensional} settings because the solution of \eqref{eq:mrc_linear} is a sparse vector. This sparsity is due to the implicit \mbox{L1-penalization in \eqref{eq:mrc}, $\B{\lambda}^{\up{T}}|\B{\mu}|$}, and the fact that usually only a small subset of features are informative in \mbox{high-dimensional} settings \citep{GhoDeb:22}. The sparsity of the MRCs' coefficients enables to carry out the learning process considering only a small subset of variables in \eqref{eq:mrc_linear} (respectively constraints in \eqref{eq:dual_mrc}). In the following, we propose an algorithm based on constraint generation methods that solves the \ac{MRC} \ac{LP} problem in \eqref{eq:mrc_linear} by iteratively selecting small subsets of features of $\Phi$.

\begin{algorithm}[H]
\captionsetup{labelfont={bf}}
\caption{Efficient learning algorithm for \mbox{0-1} \acp{MRC}}
\label{alg:efficient_mrc}
\begin{tabular}{ll}
\textbf{Input:} & \hspace{-0.4cm} $\B{\up{F}}$, $\B{\up{b}}$, $\B{\tau}$, $ \B{\lambda}$, initial subset of features $\mathcal{I}$, \\
& \hspace{-0.4cm} dual constraints' violation threshold $\epsilon$, \\
& \hspace{-0.4cm} maximum number of iterations $k_{\up{max}}$, and \\
& \hspace{-0.4cm} maximum number of features selected \\
& \hspace{-0.4cm} in each iteration $n_{\up{max}}$ \\
\textbf{Output:} & \hspace{-0.3cm}selected features \ $\mathcal{I}^{*} \subseteq \{1,2,\ldots,m\}$, \\
& \hspace{-0.3cm}optimal solution $\B{\mu}^{*} \in \mathbb{R}^{m}$, and \\
& \hspace{-0.3cm}optimal worst-case error probability $\up{R}^{*}$ \\
\end{tabular}
\begin{algorithmic}[1]
\setstretch{1.2}
\State $\B{\mu}^1 = [0,0,\ldots,0] \in \mathbb{R}^{m}$
\State LPSOLVE$\big(\mathcal{P}_{\mathcal{I}}\big)$ 
\Statex $\B{\mu}^1_1, \ \B{\mu}^1_2, \ \nu^1 \leftarrow \text{Solution of primal}$
\Statex $\B{\alpha}^1 \leftarrow \text{Solution of dual}$
\Statex $\up{R}^{1} \leftarrow$ Optimal value (worst-case error probability)
\State ${\B{\mu}^{1}}_{\mathcal{I}} = \B{\mu}^{1}_1 - \B{\mu}^{1}_2$
\State $\mathcal{J} \leftarrow$ SELECT$\big(\B{\up{F}}, \ \B{\tau}, \ \B{\lambda}, \ n_{\up{max}}, \ \mathcal{I}, \ \epsilon, \ \B{\alpha}^{1}\big)$ \label{alg:efficient_mrc:line:6}
\State $k=1$
\While{$\mathcal{J} \setminus \mathcal{I} \neq \emptyset$ and $k \leq k_{\up{max}}$}
\State $k = k +1$
\State $\B{\mu}^{k-1}_1=({\B{\mu}^{k-1}}_{\mathcal{J}})_{+}, \ \B{\mu}^{k-1}_2=(-{\B{\mu}^{k-1}}_{\mathcal{J}})_{+}$
\State LPSOLVE$\big({\mathcal{P}}_{\mathcal{J}}, \ \B{\mu}_1^{k-1}, \ \B{\mu}_2^{k-1}, \ \nu^{k-1}\big)$ \label{alg:efficient_mrc:line:11}
\Statex \hspace{0.4cm} $\B{\mu}^{k}_1, \ \B{\mu}^{k}_2, \ \nu^{k} \leftarrow$ Solution of primal
\Statex \hspace{0.4cm} $\B{\alpha}^{k} \leftarrow$ Solution of dual
\Statex \hspace{0.4cm} $\up{R}^{k} \leftarrow$ Optimal value (worst-case error probability)
\State ${\B{\mu}^{k}} = {\B{\mu}^{k-1}}$
\State ${\B{\mu}^{k}}_{\mathcal{J}} = \B{\mu}_1^{k} - \B{\mu}_2^{k}$
\State $\mathcal{I} \leftarrow \mathcal{J}$
\State $\mathcal{J}\leftarrow$ SELECT$\big(\B{\up{F}}, \ \B{\tau}, \ \B{\lambda}, \ n_{\up{max}}, \ \mathcal{I}, \ \epsilon, \ \B{\alpha}^{k}\big)$  \label{alg:efficient_mrc:line:10}
\EndWhile
\State $\mathcal{I}^{*} = \mathcal{I}$, $\B{\mu}^{*} = \B{\mu}^{k}$, $\up{R}^{*} = \up{R}^{k}$
\end{algorithmic}
\end{algorithm}

The proposed learning algorithm obtains the optimal solution $\B{\mu}^{*}$ to the original problem by iteratively solving a sequence of subproblems (see details in Algorithm~\ref{alg:efficient_mrc}). These subproblems correspond with small subsets of features selected by a constraint generation method (\mbox{SELECT} function detailed in Algorithm~\ref{alg:select}).

The subproblem of \eqref{eq:mrc_linear} corresponding to the subset of features $\mathcal{J} \subseteq \{1,2,\ldots,m\}$ is defined as 
\begin{align}
\label{eq:mrc_linear_subprob_primal}
\arraycolsep=1.4pt
\begin{array}{ccc}\mathcal{P}_{\mathcal{J}}: & \underset{\substack{\B{\mu}_1, \B{\mu}_2, \nu}}{\min} & - (\B{\tau}_{\mathcal{J}} - \B{\lambda}_{\mathcal{J}})^{\up{T}}\B{\mu}_1 + (\B{\tau}_{\mathcal{J}} + \B{\lambda}_{\mathcal{J}})^{\up{T}}\B{\mu}_2 + \nu \\
 & \text{s.t.} & \B{\up{F}}_{\mathcal{J}}(\B{\mu}_1 - \B{\mu}_2) + \nu\B{1} \preceq \B{\up{b}} \\
 & & \B{\mu}_1, \B{\mu}_2 \succeq 0 \end{array}
\end{align}
where $\B{\tau}_{\mathcal{J}}$ and $\B{\lambda}_{\mathcal{J}}$ denote the subvectors of $\B{\tau}$ and $\B{\lambda}$ corresponding to the $\mathcal{J}$ components, and $\B{\up{F}}_{\mathcal{J}}$ denotes the submatrix of $\B{\up{F}}$ corresponding to the $\mathcal{J}$ columns. In addition, the dual of \eqref{eq:mrc_linear_subprob_primal} is
\begin{align}
\label{eq:mrc_linear_subprob_dual}
\def\arraystretch{1.2}
\begin{array}{ccc}
\mathcal{D}_{\mathcal{J}}: & \underset{\B{\alpha}}{\max} & -\B{\up{b}}^{\up{T}}\B{\alpha} \\
& \text{s.t.} & \B{\tau}_{\mathcal{J}} - \B{\lambda}_{\mathcal{J}} \ \preceq \ {\B{\up{F}}_{\mathcal{J}}}^{\up{T}}\B{\alpha} \ \preceq  \ \B{\tau}_{\mathcal{J}} + \B{\lambda}_{\mathcal{J}} \\
& & \B{1}^{\up{T}}\B{\alpha} = 1, \B{\alpha} \succeq 0.
\end{array}
\end{align}
\noindent
At each iteration $k$ of Algorithm~\ref{alg:efficient_mrc} (Line~\ref{alg:efficient_mrc:line:11}), the LPSOLVE function solves such subproblems and obtains the primal solution $\B{\mu}^{k}_1, \B{\mu}^{k}_2, \nu^{k}$ and dual solution $\B{\alpha}^{k}$ along with the worst-case error probability $\up{R}^{k}$ given by the optimal value. The dual solution is used by the \mbox{SELECT} function (Line~\ref{alg:efficient_mrc:line:6} and \ref{alg:efficient_mrc:line:10}) to obtain a subsequent subset of features. The primal solution can be used by the LPSOLVE function (Line~\ref{alg:efficient_mrc:line:11}) to warm-start the optimization in the next iteration. This iterative process ends when the \mbox{SELECT} function returns a subset of features that does not contain any new feature.

Notice that at each iteration $k$, the algorithm obtains coefficients $\B{\mu}^k$, features subset $\mathcal{J}$, and worst-case error probability $\up{R}^k$. In Section~\ref{theoretical_results}, we show that such coefficients $\B{\mu}^k$ provide an \ac{MRC} with worst-case error probability $\up{R}^k$ corresponding with the uncertainty set given by features $\mathcal{J}$. In addition, we show that such worst-case error probability monotonically decreases with the number of iterations and converges to the optimal corresponding with all the features.

\subsection{Greedy feature selection}
\label{select}
The \mbox{SELECT} function determines the subset of features for the next iteration based on the previous subset and the current solution. Specifically, this greedy selection process selects features corresponding to the most violated constraints in the dual and removes features corresponding to the constraints satisfied with a positive slack (overlysatisfied constraints). Such a process aims to achieve the fastest decrease in worst-case error probability while using the smallest number of features. As described in the literature for constraint generation methods \citep{DimTsi97, DesJac}, this type of selection process can be implemented in several alternative ways such as adding any subset of violated constraints or not removing constraints. In practice, we observe that adding a subset of violated constraints and removing all the overlysatisfied constraints achieves the fastest convergence.

 \begin{algorithm}
\captionsetup{labelfont={bf}}
\caption{SELECT (greedy feature selection)}
\label{alg:select}
\begin{tabular}{ll}
\textbf{Input:} & \hspace{-0.4cm} $\B{\up{F}}$, $\B{\tau}$, $\B{\lambda}$, $n_{\up{max}}$, current subset of features $\mathcal{I}$, \\
& \hspace{-0.4cm} dual constraints' violation threshold $\epsilon$, and \\
& \hspace{-0.4cm} dual solution $\B{\alpha}$ \\
\textbf{Output:} & \hspace{-0.3cm}selected features \ $\mathcal{J} \subseteq \{1,\ldots,m\}$ \\
\end{tabular}
\begin{algorithmic}[1]
\setstretch{1.2}
\State $\B{\up{v}} = |\B{\up{F}}^{\up{T}}\B{\alpha}-\B{\tau}|-\B{\lambda}$ 
\State $\mathcal{J} = \{i: i \in \mathcal{I}, \ \up{v}_i = 0\}$ \label{alg:select:line:4}
\State $\mathcal{A} \leftarrow$ \text{ARGNMAX}$(\B{\up{v}}, n_{\up{max}})$ \label{alg:select:line:5}
\State $\set{A} = \{i: i \in \set{A}, \ \up{v}_i > \epsilon\}$
\State $\set{J} \leftarrow \set{J} \cup \set{A}$ \label{alg:select:line:6}
\end{algorithmic}
\end{algorithm}

The implementation details for function \mbox{SELECT} are provided in Algorithm~\ref{alg:select}. The vector $\B{\up{v}}$ quantifies the violations in the dual constraints. In particular, the features corresponding to the negative values of vector $\B{\up{v}}$ (oversatisfied constraints) are removed, and the features corresponding to the $n_{\up{max}}$ largest positive values of vector $\B{\up{v}}$ are added. Moreover, the selection of features is also restricted by \mbox{hyperparameter} $\epsilon \geq 0$ that represents the minimum violation for dual constraints.

\subsection{Computational complexity}
\label{computation_complexity}
The computational complexity of Algorithm~\ref{alg:efficient_mrc} is given by the number of iterations and the complexity per iteration that depend on the maximum number of features selected in each iteration $n_{\up{max}}$ and the hyperparameter $\epsilon$. Decreasing $n_{\up{max}}$ results in a reduced complexity per iteration at the expense of an increased number of iterations. Increasing $\epsilon$ can decrease the number of iterations and complexity per iteration at the expense of achieving approximate solutions. For instance, if $n_{\up{max}}=m$ and $\epsilon=0$ the algorithm finds the exact solution in only one iteration but the complexity of such iteration is large. In Section 4, we show that in practice $\epsilon=0.0001$ and $n_{\text{max}} \in [50,500]$ obtains efficient and accurate learning of \acp{MRC}.

The number of iterations can be further reduced by using an adequate choice for the initial subset of features. This initial subset of features $\mathcal{I}$ can be obtained efficiently using a simple approach for coarse feature selection such as correlation screening \citep{tibshirani2012strong} or few iterations of a first order optimization method \citep{MazRomGrun:23}. Similarly, the complexity per iteration is reduced by using a warm-start in LPSOLVE. In the proposed algorithm, a warm-start for iteration $k$ is obtained directly from the previous solution $\B{\mu}^{k-1}$ of the \ac{LP} at iteration $k-1$. Note that this warm-start is a basic feasible solution for the \ac{LP} at iteration $k$ since it is obtained by removing and adding features for which the corresponding coefficients at $k-1$ are zero.

\subsection{Theoretical analysis of the algorithm}
\label{theoretical_results}
The following shows the theoretical properties of the proposed learning algorithm. In particular, the algorithm's iterations provide a sequence of \acp{MRC} with decreasing worst-case error probabilities that converges to the \ac{MRC} corresponding with all the features.

The following theorem shows that the algorithm provides a decreasing sequence of worst-case error probabilities that can provide an upper bound to the classification error of the corresponding \acp{MRC}.
\begin{theorem} \label{th:mono_dec}
Let $\B{\mu}^k$ and $\up{R}^k$ be the coefficients and worst-case error probability obtained by the proposed algorithm at iteration k. If $\up{h}^k$ is the \ac{MRC} given by $\B{\mu}^k$, then we have that 
\begin{equation}
\label{ineq:bound}
\mathcal{R}(\up{h}^{k}) \leq \up{R}^{k} + (|\mathbb{E}_{\up{p}^{*}}\{\Phi(x,y)_{\set{J}}\} - \B{\tau}_{\set{J}}| - \B{\lambda}_{\set{J}})^{\up{T}}|\B{\mu}_{\set{J}}^{k}|
\end{equation}
where $\mathcal{J}$ is the subset of features used at iteration $k$. In addition, the algorithm provides \acp{MRC} with decreasing worst-case error probabilities, that is,
\begin{equation}\label{ineq:mono_dec}
\up{R}^{k+1} \leq \up{R}^{k}, \ k=1, 2,\ldots
\end{equation}
\end{theorem}
\begin{proof}
At each intermediate step, $\B{\mu}^k$ and $\up{R}^k$ are the coefficients and worst-case error probability of the MRC corresponding with uncertainty set
\begin{align}
\mathcal{U}^k=\{\up{p} \in \Delta (\mathcal{X} \times \mathcal{Y}) : |\mathbb{E}_{\up{p}^{*}}\{\Phi(x,y)_{\set{J}}\}-\B{\tau}_{\set{J}}| \preceq \B{\lambda}_{\set{J}}\}
\end{align}
using Theorem~\ref{th:linear_form} for the subproblem \eqref{eq:mrc_linear_subprob_primal}. Therefore, inequality~\eqref{ineq:bound} is obtained using the bounds for error probabilities for \acp{MRC} in \cite{MazRomGrun:23}. The second result is obtained since the warm-start at iteration $k+1$ is a feasible solution with an objective value that equals the worst-case error probability $\up{R}^{k}$ at iteration $k$. The \mbox{warm-start} is feasible at $k+1$ and has value $\up{R}^{k}$ because it is obtained from the solution at iteration $k$ by removing and adding features corresponding to zero coefficients.
\end{proof}
Inequality \eqref{ineq:bound} bounds the classification error of the \ac{MRC} classifier at any intermediate iteration. Specifically, the worst-case error probability $\up{R}^{k}$ at any iteration $k$ corresponding to a feature subset $\mathcal{J}$ directly provides such a bound if the error in the mean vector estimate is not underestimated, i.e., $\B{\lambda}_{\mathcal{J}} \succeq |\mathbb{E}_{\up{p}^{*}}\{\Phi(x,y)_{\set{J}}\} - \B{\tau}_{\mathcal{J}}|$. In particular, if $\B{\lambda}_{\set{J}}$ is a confidence vector with coverage probability $1-\delta$, that is, $\mathbb{P}\{|\mathbb{E}_{\up{p}^{*}}\{\Phi(x,y)_{\set{J}}\} - \B{\tau}_{\set{J}}| \preceq \B{\lambda}_{\set{J}}\} \geq 1-\delta$, then
\begin{equation}
\label{ineq:subproblem_bound}
\mathcal{R}(\up{h}^{k}) \leq \up{R}^{k}
\end{equation}
with probability at least $1-\delta$. In other cases, $\up{R}^{k}$ still provides approximate bounds as long as the underestimation $|\mathbb{E}_{\up{p}^{*}}\{\Phi(x,y)_{\set{J}}\} - \B{\tau}_{\mathcal{J}}| - \B{\lambda}_{\mathcal{J}}$ is not substantial. Similar bounds also hold for the sequence of deterministic classification rules due to inequality in \eqref{eq:cl_risk}.

The next theorem shows the convergence of the sequence of \acp{MRC} obtained by the proposed algorithm to the \ac{MRC} corresponding with all the features.\begin{theorem}
\label{th:convergence}
Let $\B{\mu}^*$ and $\up{R}^{*}$ be the \ac{MRC} coefficients and worst-case error probability obtained by solving \eqref{eq:mrc_linear} using all the features. If $\up{R}^{k}$, $k=1, 2, \ldots$, is the sequence of worst-case error probabilities obtained by the proposed algorithm using dual constraints' violation threshold $\epsilon$. Then, there exists $k_0\geq1$ such that 
\begin{equation}\label{ineq:convergence}
\up{R}^{*} \leq \up{R}^{k} \leq \up{R}^{*} + \epsilon\|\B{\mu}^{*}\|_1
\end{equation}
for any $k \geq k_0$.
\begin{proof}
The first inequality follows by noting that the feasible set of \eqref{eq:dual_mrc} corresponding with all the features is contained in the feasible set of any subproblem given a subset of features. In the following, we prove the second inequality. 

Let $k_0 \geq 1$ be an iteration in which \mbox{SELECT} function in Algorithm~\ref{alg:efficient_mrc} does not add any new feature. Such a case occurs after a finite number of iterations due to the properties of constraint generation methods \citep{DimTsi97}. If $\B{\alpha}^{k_0}$ is the dual solution obtained at iteration $k_0$ with subset of features $\set{J}$, then we have that 
\begin{align}\label{ineq:final_violation_1}
\begin{array}{c}
\B{\tau}_{\mathcal{J}} - \B{\lambda}_{\mathcal{J}} \ \preceq \ \B{\up{F}}_{\mathcal{J}}\B{\alpha}^{k_0} \ \preceq  \ \B{\tau}_{\mathcal{J}} + \B{\lambda}_{\mathcal{J}} 
\end{array}
\end{align}
\begin{align}\label{ineq:final_violation_2}
\B{\tau}_{{\mathcal{J}}^{\up{c}}} - \B{\lambda}_{{\mathcal{J}}^{\up{c}}} - \epsilon\B{1} \ \preceq \ \B{\up{F}}_{{\mathcal{J}}^{\up{c}}}\B{\alpha}^{k_0} \ \preceq  \ \B{\tau}_{{\mathcal{J}}^{\up{c}}} + \B{\lambda}_{{\mathcal{J}}^{\up{c}}} + \epsilon\B{1}.
\end{align}
On combining \eqref{ineq:final_violation_1} and \eqref{ineq:final_violation_2}, we have
\begin{align}\label{ineq:final_violation}
\begin{array}{c}
\B{\tau} - \B{\lambda} - \epsilon\B{1} \ \preceq \ \B{\up{F}}\B{\alpha}^{k_0} \ \preceq  \ \B{\tau} + \B{\lambda} + \epsilon\B{1}.
\end{array}
\end{align}
Now, consider the following primal problem corresponding to the dual \eqref{eq:dual_mrc} with constraints as in \eqref{ineq:final_violation}.
\begin{align}
\arraycolsep=1.4pt
\label{eq:mrc_linear_updated}
\begin{array}{cc}
\underset{\substack{\B{\mu}_1, \B{\mu}_2, \nu}}{\min} & - (\B{\tau} - \B{\lambda} - \epsilon\B{1})^{\up{T}}\B{\mu}_1 + (\B{\tau} + \B{\lambda} + \epsilon\B{1})^{\up{T}}\B{\mu}_2 + \nu \\ 
 \text{s.t.} & \B{\up{F}}^{\up{T}}(\B{\mu}_1 - \B{\mu}_2) - \nu\B{1} \preceq \B{\up{c}} \\
& \B{\mu}_1, \B{\mu}_2 \succeq 0.
\end{array}
\end{align}
If $\B{\mu}^{*}_1$, $\B{\mu}^{*}_2$, and $\nu^{*}$ correspond to the \ac{MRC} coefficients obtained by solving \eqref{eq:mrc_linear} using all the features, then it is also a feasible solution to the problem \eqref{eq:mrc_linear_updated} since they have the same feasible set. Therefore, using weak duality (Theorem 4.3 in \cite{DimTsi97}), we have that
\begin{align}
-\B{\up{b}}^{\up{T}}\B{\alpha}^{k_0} \leq - (\B{\tau} - \B{\lambda} - \epsilon\B{1})^{\up{T}}\B{\mu}^{*}_1 + (\B{\tau} + \B{\lambda} + \epsilon\B{1})^{\up{T}}\B{\mu}^{*}_2 + \nu^{*}
\end{align}
since $\B{\alpha}^{k_0}$ is a feasible solution for the dual of \eqref{eq:mrc_linear_updated}. Therefore,
\begin{align}
\up{R}^{k_0} \leq \up{R}^{*} + \epsilon\B{1}(\B{\mu}^{*}_1 + \B{\mu}^{*}_2) = \up{R}^{*} + \epsilon\|\B{\mu}^{*}\|_{1}
\end{align}
since $\B{\mu}^{*}_1 = (\B{\mu}^{*})_{+}$ and $\B{\mu}^{*}_2 = (-\B{\mu}^{*})_{+}$. Hence, the second inequality holds for any $k \geq k_0$ since $\up{R}^{k} \leq \up{R}^{k_0}$ due to the monotonic decrease of the worst-case error probability shown in Theorem~\ref{th:mono_dec}.
\end{proof}
\end{theorem}
Inequality~\eqref{ineq:convergence} shows that if $\epsilon=0$ the algorithm finds the \ac{MRC} corresponding to all the features. In other cases, it finds an approximate solution whose accuracy depends on the hyperparameter $\epsilon$. Therefore, the hyperparameter $\epsilon$ can serve to reduce the complexity of the algorithm while obtaining near optimal solutions. In the next section, we further analyze the effect of hyperparameter $\epsilon$ through numerical experiments using multiple real datasets.
\begin{table}
 \captionsetup{labelfont={it}, labelsep=period, font=small}
                         \caption{High-dimensional data sets.}
    \vskip -0.15in
     \label{tb:datasets}
\setstretch{1.2}
\begin{center}
\scalebox{0.85}{\begin{tabular}{|c|c|c|c|c|}
\hline
Number & Data set & \multicolumn{1}{c|}{Variables} & \multicolumn{1}{c|}{Samples} & \multicolumn{1}{c|}{Classes} \\ \hline
1 & Arcene & 10000  & 200 & 2 \\
2 & Colon & 2000 & 62 & 2  \\
3 & CLL\textunderscore SUB\textunderscore 111 & 11340 & 111 & 3 \\
4 & Dorothea & 100000 & 1150 & 2\\
5 & GLI\textunderscore 85 & 22283 & 85 & 2\\
6 & GLIOMA & 4434 & 50 & 4 \\
7 & Leukemia & 7129 & 72 & 2 \\
8 & Lung & 12600 & 203 & 5\\
9 & MLL & 12582 & 72 & 3\\
10 & Ovarian & 15154 & 253 & 2 \\
11 & Prostate\textunderscore GE & 5966 & 102 & 2\\
12 & SMK\textunderscore CAN\textunderscore 187 & 19993 & 187 & 2\\
13 & TOX\textunderscore 171 & 5748 & 171 & 4 \\
 \hline
\end{tabular}}
\end{center}
   \vskip -0.25in
\end{table}

\begin{figure*}
 \centering
      \begin{subfigure}[b]{0.49\textwidth}
         \centering
           \psfrag{k}[c][t][0.9]{Number of iterations $k$}
         \psfrag{R1}[c][t][0.9]{$\up{R}^{k} - \up{R}^{*}$}
         \psfrag{a}[][l][0.7]{$10^{0}$ \ \ \ }
         \psfrag{b}[][l][0.7]{$10^{-3}$ \ \ \ }
         \psfrag{c}[][l][0.7]{$10^{-6}$ \ \ \ }
         \psfrag{d}[][l][0.7]{$10^{-9}$ \ \ \ } 
          \psfrag{0}[][][0.7]{0}
           \psfrag{10}[][][0.7]{10}
            \psfrag{20}[][][0.7]{20}
             \psfrag{5}[][][0.7]{}
              \psfrag{15}[][][0.7]{}
               \psfrag{25}[][][0.7]{}
         \psfrag{E = 0.01}[l][l][0.8]{$\epsilon = 0.01$}
         \psfrag{E = 0.001}[l][l][0.8]{$\epsilon = 0.001$}
         \psfrag{E = 0.000000000001}[l][l][0.8]{$\epsilon = 0.0001$}
         \includegraphics[width=\textwidth]{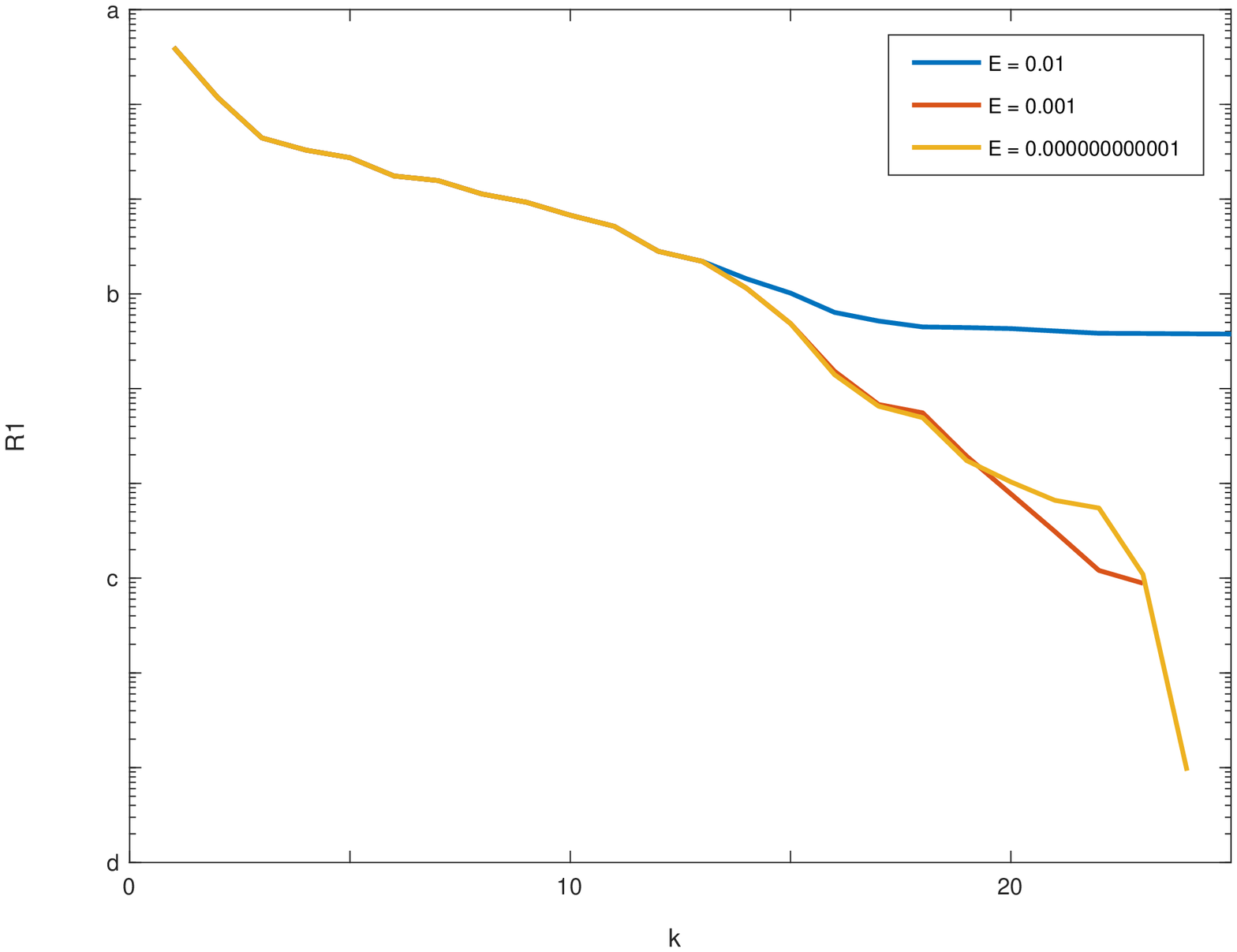}
         \captionsetup{font=small}
         \caption{Ovarian}
         \label{ovarian_decreasing_bounds}
     \end{subfigure}
      \begin{subfigure}[b]{0.49\textwidth}
         \centering
         \psfrag{k}[c][t][0.9]{Number of iterations $k$}
         \psfrag{R1}[c][t][0.9]{$\up{R}^{k} - \up{R}^{*}$}
         \psfrag{a}[][l][0.7]{$10^{0}$ \ \ \ }
         \psfrag{b}[][l][0.7]{$10^{-4}$ \ \ \ }
         \psfrag{c}[][l][0.7]{$10^{-8}$ \ \ \ }
        \psfrag{E = 0.01}[l][l][0.8]{$\epsilon = 0.01$}
         \psfrag{E = 0.001}[l][l][0.8]{$\epsilon = 0.001$}
         \psfrag{E = 0.000000000001}[l][l][0.8]{$\epsilon = 0.0001$}
         \psfrag{0}[][][0.7]{0}
          \psfrag{10}[][][0.7]{10}
           \psfrag{20}[][][0.7]{20}
           \psfrag{40}[][][0.7]{40}
         \includegraphics[width=\textwidth]{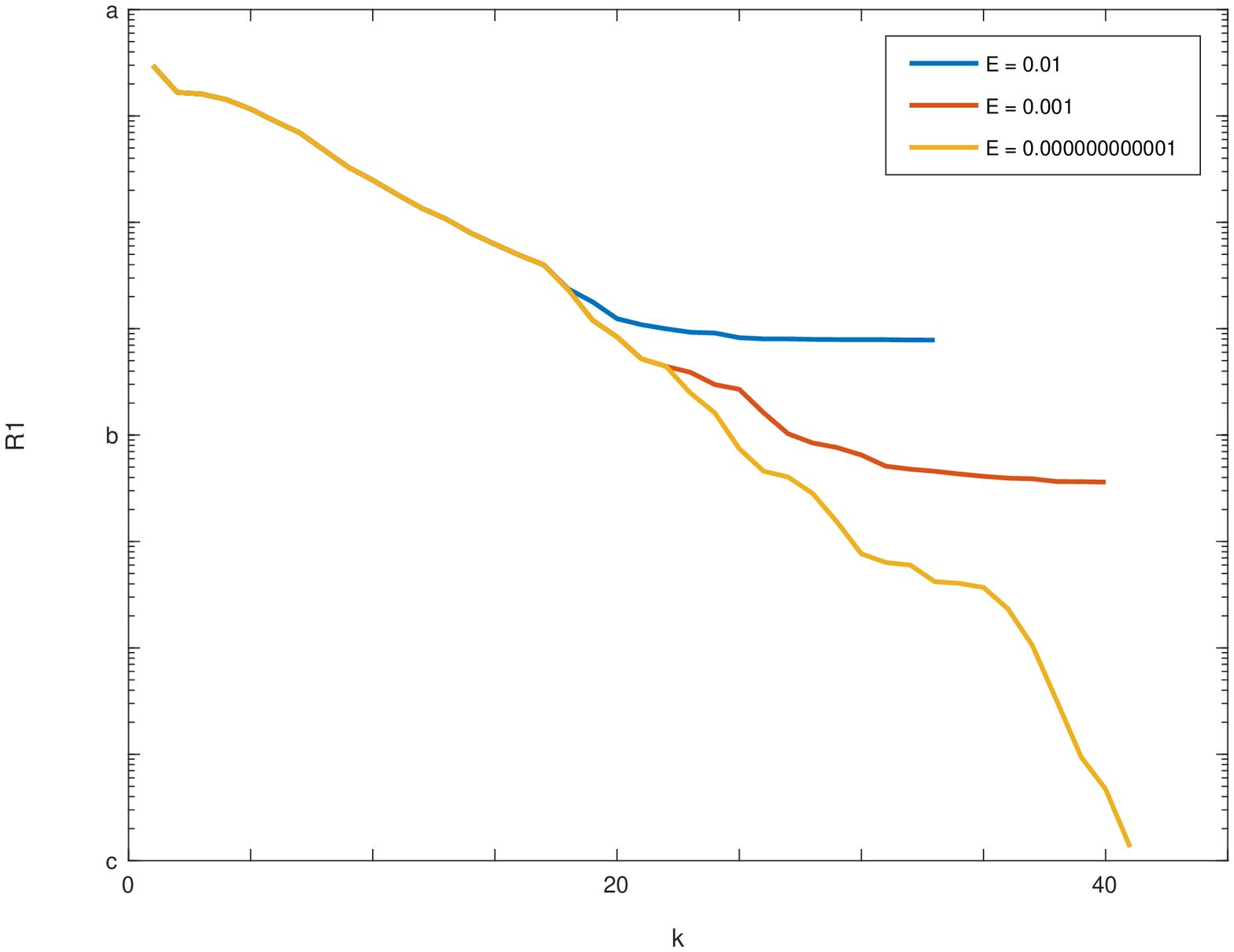}
         \captionsetup{font=small}
         \caption{Arcene}
         \label{arcene_decreasing_bounds}
     \end{subfigure} 
       \vskip -0.1in
    \captionsetup{labelfont={it}, labelsep=period, font=small}
    \caption{Monotonic decrease of the worst-case error probability $\up{R}^{k}$ increasing {\color{black}the} number of iterations $k$.}
    \label{fig:decreasing_upper_bound}
\end{figure*}

\begin{figure}
\centering
 \psfrag{arcene}[l][l][0.8]{Arcene}
 \psfrag{12345678912345678912345678}[l][l][0.8]{Prostate\textunderscore GE}
 \psfrag{smkcan}[l][l][0.8]{SMK\textunderscore CAN\textunderscore 187}
 \psfrag{ovarian}[l][l][0.8]{Ovarian}
 \psfrag{0.3}[][][0.7]{0.3}
 \psfrag{0.4}[][][0.7]{0.4}
 \psfrag{0.2}[][][0.7]{0.2}
 \psfrag{0.1}[][][0.7]{0.1}
  \psfrag{0}[][][0.7]{0}
  \psfrag{a}[][][0.7]{$10^{0}$}
  \psfrag{b}[][][0.7]{$10^{1}$}
  \psfrag{c}[][][0.7]{$10^{2}$}
    \psfrag{d}[][][0.8]{$10^{3}$}
 \psfrag{nmax}[c][t][0.9]{Maximum number of features selected $n_{\up{max}}$}
 \psfrag{Reltime}[][][0.9]{Relative time}
\includegraphics[width=0.49\textwidth]{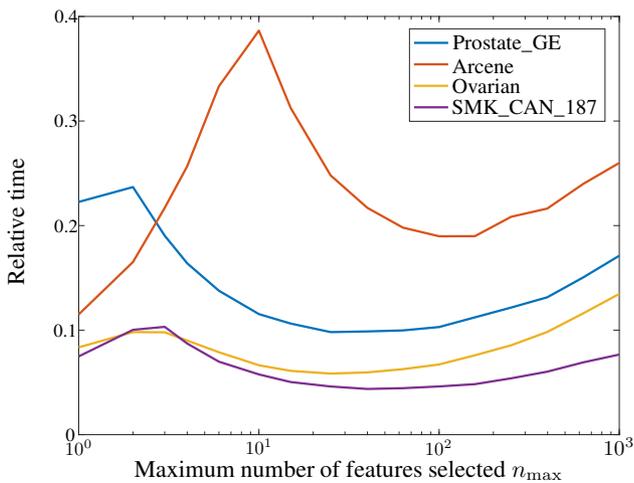}
\captionsetup{font=small}
\caption{Improved efficiency of \mbox{MRC-CG} over \mbox{MRC-LP} for different values of maximum number of features selected $n_{\up{max}}$}
\label{fig:efficiency_based_on_feature_selection}
\end{figure}

\section{Results} \label{results}
In this section, we present a set of experimental results that analyze the effect of hyperparameters and quantify the training time improvements in practice. In addition, we describe the benefits of having the worst-case error probabilities and show that the presented algorithm performs efficient feature selection. The proposed Algorithm~\ref{alg:efficient_mrc} (MRC-CG) is compared with the \ac{LP} formulation for \acp{MRC} in \ref{eq:mrc_linear} using all the features (\mbox{MRC-LP}), and the constraint generation method for L1-\acp{SVM} (\mbox{SVM-CG}) \citep{DedAntEtal:22}. In addition, the quality of the feature selection is compared with \mbox{SVM-CG}, \ac{RFE} for \acp{SVM} \citep{GuyIsaEtal:02}, \ac{MRMR}, and \ac{ANOVA} \citep{DinPen:05,PenLonDin:05}. The experimental results are obtained using 13 real-world \mbox{high-dimensional} datasets as shown in Table~\ref{tb:datasets}. The datasets 4 and 10 can be obtained from the UCI \citep{DuaGra:2019} repository and the remaining from \url{https://jundongl.github.io/scikit-feature/datasets.html}. The implementation of the algorithm proposed is available in the library MRCpy \citep{BonMazPer:23} and the experimental setup in \url{https://github.com/MachineLearningBCAM/Constraint-Generation-for-MRCs}.

\begin{figure*}
     \centering
     \begin{subfigure}[b]{0.49\textwidth}
         \centering
         \psfrag{nfeat}[c][t][0.9]{Number of features}
	 \psfrag{time}[c][t][0.9]{Training time (in secs)}
	 \psfrag{123456789123456789}[l][l][0.8]{MRC-CG}
 	\psfrag{svc-cg}[l][l][0.8]{SVM-CG}
	\psfrag{lp-mrc}[l][l][0.8]{MRC-LP} 	
         \psfrag{2}[][][0.7]{2}
	\psfrag{4}[][][0.7]{4}
	\psfrag{6}[][][0.7]{6}
	\psfrag{0}[][][0.7]{0}
	\psfrag{a}[][][0.7]{$10^{2}$}
	\psfrag{b}[][][0.7]{$10^{3}$}
	\psfrag{c}[][][0.7]{$10^{4}$}
         \includegraphics[width=\textwidth]{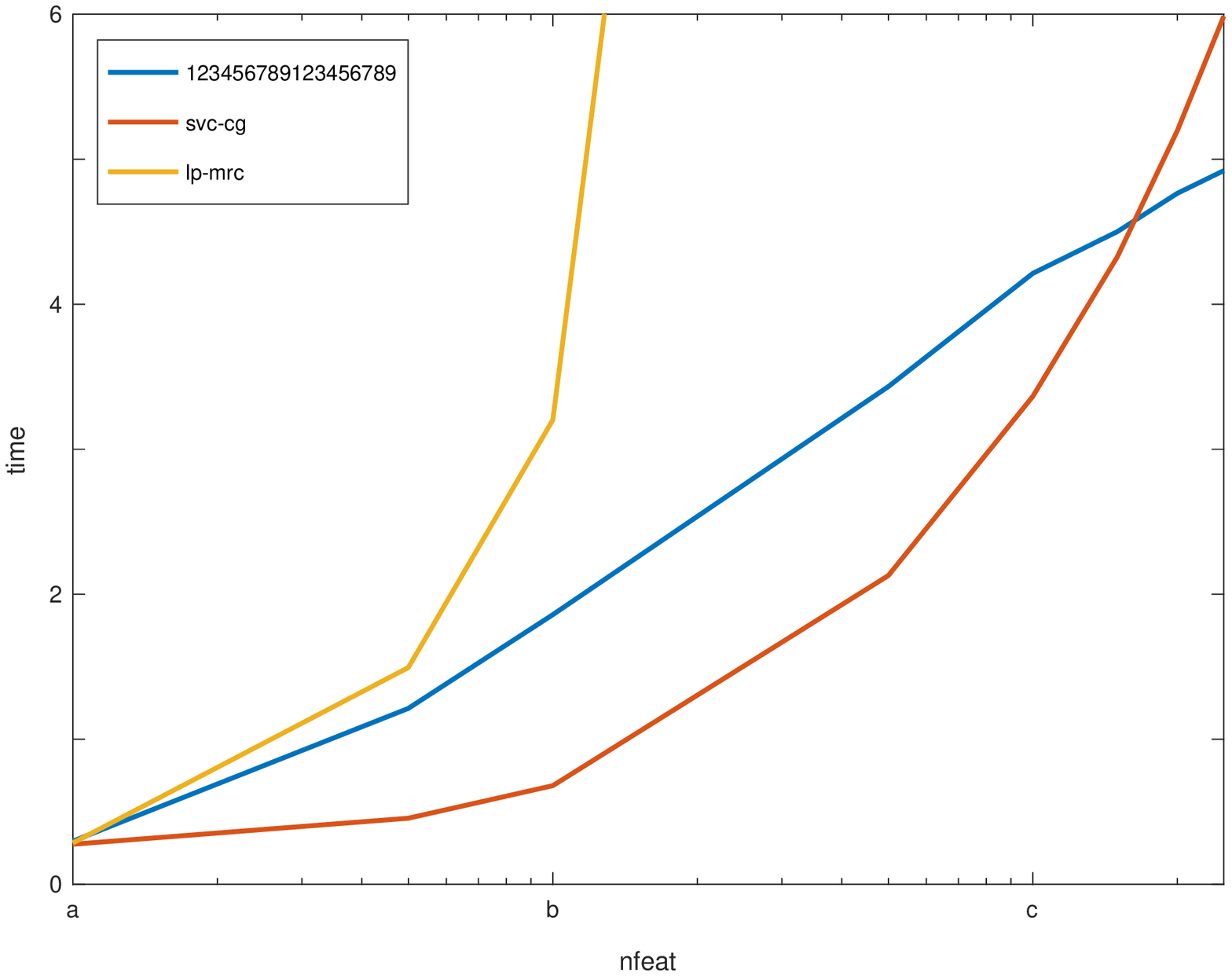}
         \captionsetup{font=small}
         \caption{Ovarian}
         \label{ovarian}
     \end{subfigure}
      \begin{subfigure}[b]{0.49\textwidth}
         \centering
        \psfrag{nfeat}[c][t][0.9]{Number of features}
	 \psfrag{time}[c][t][0.9]{Training time (in secs)}
	\psfrag{123456789123456789}[l][l][0.8]{MRC-CG}
 	\psfrag{svc-cg}[l][l][0.8]{SVM-CG}
	\psfrag{lp-mrc}[l][l][0.8]{MRC-LP} 	
	\psfrag{2}[][][0.7]{2}
	\psfrag{4}[][][0.7]{4}
	\psfrag{6}[][][0.7]{6}
	\psfrag{0}[][][0.7]{0}
	\psfrag{a}[][][0.7]{$10^{2}$}
	\psfrag{b}[][][0.7]{$10^{3}$}
	\psfrag{c}[][][0.7]{$10^{4}$}
         \includegraphics[width=\textwidth]{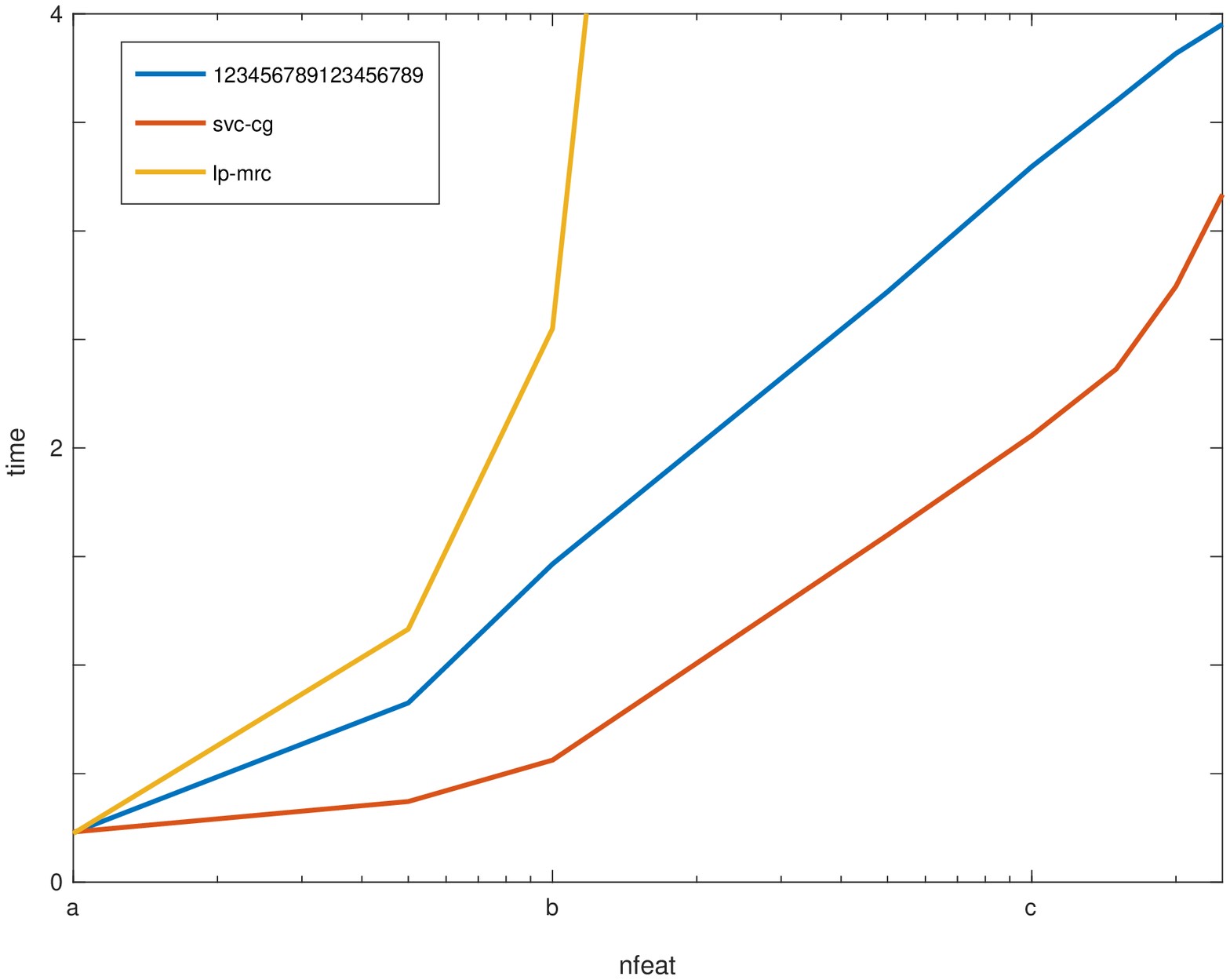}
         \captionsetup{font=small}
         \caption{Arcene}
         \label{arcene}
     \end{subfigure}
    \vskip -0.1in
    \captionsetup{labelfont={it}, labelsep=period, font=small}
    \caption{Comparison of training times (in secs) of \mbox{MRC-CG}, \mbox{MRC-LP}, and \mbox{SVM-CG} for increasing number of \ac{RFF}.}
    \label{fig:times}
\end{figure*}

\paragraph{The hyperparameter $\boldsymbol{\epsilon}$:}
We present results that show the influence of $\epsilon$, and illustrate in practice the theoretical properties of MRC-CG using datasets "Ovarian" and "Arcene". The worst-case error probability $\up{R}^{*}$ is obtained by solving \eqref{eq:mrc_linear} using all the features and the sequence of worst-case error probabilities $\up{R}^{k}$ is obtained by \mbox{MRC-CG} for increasing number of iterations $k$. In Figure~\ref{fig:decreasing_upper_bound}, we show the convergence of $\up{R}^{k}$ to $\up{R}^{*}$ using \mbox{$\epsilon = \{0.01, 0.001, 0.0001\}$} and $n_{\up{max}}=100$. In practice, we observe that $\up{R}^{k}$ converges to $\up{R}^{*}$ with differences atmost in the order of $10^{-3}$ even for $\epsilon=0.01$. In addition, the results show that $\up{R}^{k}$ is monotonically decreasing (as shown in Theorem~\ref{th:mono_dec}) and achieves significant convergence in few iterations. We also observe that the smaller values of $\epsilon$ lead to more accurate results, as shown in Theorem~\ref{th:convergence}. On average, highly accurate results are obtained in 20 iterations using $\epsilon=0.0001$.

\paragraph{The hyperparameter $\boldsymbol{n_{\up{max}}}$:}
We present results that show the influence of the hyperparameter $n_{\up{max}}$ on the training time of \mbox{MRC-CG} in comparison with \mbox{MRC-LP} using datasets "Prostate\textunderscore GE", "Arcene", "Ovarian" and "SMK\textunderscore CAN\textunderscore 187". The parameters of \mbox{MRC-CG} are taken as \mbox{$\epsilon=0.0001$} and $k_{\up{max}}=20$, and the training times are averaged over 50 random repetitions using 90\% of the data. Figure~\ref{fig:efficiency_based_on_feature_selection} presents the relative time computed as the time taken by \mbox{MRC-CG} over the time taken by \mbox{MRC-LP} for values of $n_{\up{max}}$ in the range of 1 to 1,500. The figure shows the trade-off due to the choice of hyperparameter $n_\up{max}$. Increasing the value of $n_\up{max}$ decreases the number of iterations for convergence at the expense of increasing the complexity per iteration. In practice, with $n_\up{max}$ in the interval [50, 500], we obtain a good compromise between the number of iterations required for convergence and the complexity per iteration. On average, for $n_\up{max}=100$, MRC-CG is 10 times faster than MRC-LP. Based on these results, in the remainder of this section we take $n_\up{max}=100$, $\epsilon=0.0001$, and $k_\up{max}=10$.

\paragraph{Scalability with increasing number of features:}
We present results to compare the scalability of \mbox{MRC-CG} with \mbox{MRC-LP}, and \mbox{SVM-CG} for increasing number of features on the datasets "Ovarian" and "Arcene". In particular, the results show the effect of the number of features on the training time of \mbox{MRC-CG} and compare it with \mbox{MRC-LP} and \mbox{SVM-CG}. The number of features range from 100 to 25,000 obtained using \acp{RFF}. Figure~\ref{fig:times} presents the average training times for \mbox{MRC-CG}, \mbox{MRC-LP}, and \mbox{SVM-CG} over 20 random repetitions. We observe that \mbox{MRC-CG} is faster than \mbox{MRC-LP} for all cases and improves the scaling with the number of features significantly better than \mbox{MRC-LP}. In addition, the training times of \mbox{MRC-CG} are competitive with training times of \mbox{SVM-CG}, especially for large number of features.

\begin{table}
 \captionsetup{labelfont={it}, labelsep=period, font=small}
                         \caption{Comparison of training times (in secs) of \mbox{MRC-CG} and \mbox{SVM-CG} along with the worst-case error probabilities $\up{R}^{*}$ and \mbox{cross-validated} error estimate for \mbox{MRC-CG} using multiple high dimensional datasets.}
    \vskip -0.15in
     \label{tb:upper_bound_errors}
\setstretch{1.2}
\begin{center}
\scalebox{0.85}{\begin{tabular}{|c|c|c|c|c|}
\hline
\multirow{2}{*}{Dataset} & \multicolumn{3}{|c|}{\ac{MRC}-CG} & \ac{SVM}-CG \\
\cline{2-5}
& $\up{R}^{*}$  & Error & Time (in secs) & Time (in secs) \\
\hline
1      							& 0.20 			& 0.30 $\pm$ 0.10 	& 3.56 $\pm$ 0.04   	& 8.25 $\pm$ 0.54 \\
2 							& 0.24 			& 0.13 $\pm$ 0.12	& 0.20 $\pm$ 0.02  	& 0.58 $\pm$ 0.07 \\
3 							& 0.16 			& 0.15 $\pm$ 0.09 	& 6.84 $\pm$ 0.04 	& 3.55 $\pm$ 0.35 \\
4   							& 0.07			& 0.07 $\pm$ 0.01 	& 31.4 $\pm$ 0.64 	& 220 $\pm$ 10.5 \\
5      							& 0.13			& 0.18 $\pm$ 0.17    & 1.50 $\pm$ 0.07 	& 1.36 $\pm$ 0.18 \\
6 							& 0.18 			& 0.38 $\pm$ 0.20 	& 4.55 $\pm$ 0.02  	& 2.58 $\pm$ 0.24 \\
7  						 	& 0.11 			& 0.02 $\pm$ 0.05	& 0.70 $\pm$ 0.05 	& 0.37 $\pm$ 0.08 \\
8 							& 0.15			& 0.05 $\pm$ 0.05 	& 44.9 $\pm$ 0.96 	& 16.6 $\pm$ 0.76 \\
9 							& 0.11 			& 0.04 $\pm$ 0.06 	& 5.06 $\pm$ 0.03 	& 1.45 $\pm$ 0.20 \\
10     						& 0.10 			& 0.00 $\pm$ 0.00 	& 3.86 $\pm$ 0.26 	& 3.35 $\pm$ 0.32 \\
11							& 0.14 			& 0.08 $\pm$ 0.07 	& 0.79 $\pm$ 0.09 	& 1.00 $\pm$ 0.10 \\
12 							& 0.23 			& 0.25 $\pm$ 0.11 	& 2.65 $\pm$ 0.13    & 8.9 $\pm$ 0.63 \\
13 							& 0.18 			& 0.03 $\pm$ 0.03 	& 17.7 $\pm$ 0.21 	& 10.5 $\pm$ 0.22 \\
\hline
\end{tabular}}
\end{center}
 \vskip -0.25in
\end{table}

\paragraph{Comparison using real-world datasets:}
Table~\ref{tb:upper_bound_errors} presents results to compare the average training times for \mbox{MRC-CG} and \mbox{SVM-CG} using the 13 high-dimensional datasets. We observe that \mbox{MRC-CG} is competitive with the \mbox{SVM-CG} in terms of training times. In addition, Table~\ref{tb:upper_bound_errors} presents worst-case error probability $\up{R}^{*}$ given by \mbox{MRC-CG} along with the 10-fold cross-validated error estimate. We observe that the worst-case error probability $\up{R}^{*}$ obtained using all the data at training can provide an error assessment for \mbox{MRC-CG} without requiring cross-validation. For instance, the $\up{R}^{*}$ is in the confidence interval of the estimated error in 8 out of 13 datasets, and is an upper bound to the confidence interval of the estimated error in the remaining datasets. This worst-case error probability is particularly useful in the addressed setting since the scarcity of training samples results in a high variability for the error estimates based on cross-validation, as shown in the table.

\begin{table*}
 \captionsetup{labelfont={it}, labelsep=period, font=small}
                         \caption{Comparison of number of features selected and the error estimate obtained for \ac{LR} and \ac{DT} using \mbox{MRC-CG} in comparison with multiple methods for feature selection.}
    \vskip -0.15in
     \label{tb:feature_selection}
\setstretch{1.2}
\begin{center}
\scalebox{0.77}{\begin{tabular}{|c|c|c|c|c|c|c|c|c|c|c|c|c|}
\hline
\multicolumn{1}{|c|}{\multirow{3}{*}{Dataset}} & \multicolumn{3}{c|}{MRC-CG} & \multicolumn{3}{c|}{SVM-CG} & \multicolumn{2}{c|}{RFE} & \multicolumn{2}{c|}{MRMR} & \multicolumn{2}{c|}{ANOVA} \\
 \cline{2-13}
  & \multicolumn{2}{|c|}{Error} & No. of & \multicolumn{2}{|c|}{Error} & No. of & \multicolumn{2}{|c|}{Error} & \multicolumn{2}{|c|}{Error} & \multicolumn{2}{|c|}{Error} \\
  \cline{2-3} \cline{5-6} \cline{8-13}
  & LR & DT & features & LR & DT & features & LR & DT & LR & DT & LR & DT \\
 \hline
1 	& .29 $\pm$ .10  & .33 $\pm$ .06   & 172 $\pm$ 3 & .27 $\pm$ .11 & .41 $\pm$ .09 & 150 $\pm$ 3 & .28 $\pm$ .05 & .31 $\pm$ .12 & .31 $\pm$ .10 & .39 $\pm$ .12 & .27 $\pm$ .07 & .31 $\pm$ .09 \\

2 	& .20 $\pm$ .10  & .19 $\pm$ .10  & \ \ 33 $\pm$ 2 & .19 $\pm$ .10 & .22 $\pm$ .13 & \ \ 32 $\pm$ 1 & .24 $\pm$ .11 & .20 $\pm$ .10 & .22 $\pm$ .10 & .29 $\pm$ .17 & .22 $\pm$ .08 & .26 $\pm$ .16 \\

5 	& .23 $\pm$ .13 & .17 $\pm$ .12 & \ \ 76 $\pm$ 0 & .14 $\pm$ .07 & .29 $\pm$ .11 & \ \ 51 $\pm$ 2 & .20 $\pm$ .11 & .31 $\pm$ .13 & .14 $\pm$ .12 & .20 $\pm$ .11 & .13 $\pm$ .10 & .20 $\pm$ .09 \\

7 	& .04 $\pm$ .06   & .12 $\pm$ .09 & \ \  63 $\pm$ 1 & .02 $\pm$ .05 & .15 $\pm$ .12 & \ \ 37 $\pm$ 2 & .05 $\pm$ .06 & .16 $\pm$ .18 & .04 $\pm$ .06 & .15 $\pm$ .16 & .04 $\pm$ .06 & .16 $\pm$ .12 \\

10	& .00 $\pm$ .00    & .02 $\pm$ .02     & 118 $\pm$ 6 & .00 $\pm$ .00 & .01 $\pm$ .02 & \ \ 33 $\pm$ 2 & .00 $\pm$ .00 & .03 $\pm$ .02 & .00 $\pm$ .00 & .02 $\pm$ .04 & .00 $\pm$ .00 & .01 $\pm$ .02 \\

11	&  .08 $\pm$ .06  & .16 $\pm$ .08 &  \ \ 77 $\pm$ 2 & .10 $\pm$ .08 & .13 $\pm$ .10 & \ \ 45 $\pm$ 1 & .10 $\pm$ .10 & .18 $\pm$ .12 & .08 $\pm$ .07 & .14 $\pm$ .12 & .08 $\pm$ .06 & .16 $\pm$ .14 \\

12    & .29 $\pm$ .12 & .41 $\pm$ .11 & 135 $\pm$ 3 & .27 $\pm$ .09 & .37 $\pm$ .09 & 127 $\pm$ 4 & .33 $\pm$ .08 & .41 $\pm$ .11 & .32 $\pm$ .12 & .39 $\pm$ .09 & .30 $\pm$ .08 & .37 $\pm$ .09 \\

\hline
\end{tabular}}
\end{center}
\vskip -0.25in
\end{table*}

\paragraph{Feature selection:}
Table~\ref{tb:feature_selection} presents results to compare the \mbox{MRC-CG} as a feature selection approach with \mbox{SVM-CG}, \mbox{RFE}, \ac{MRMR}, and \ac{ANOVA}. The table shows the 10-fold \mbox{cross-validated} errors for \acf{LR} and \acf{DT} using the features selected by the different methods. We assess the quality of features selected based on the errors and the number of features selected. Note that the number of features selected by \mbox{RFE}, \ac{MRMR}, and \ac{ANOVA} were set to the number of features selected by \mbox{SVM-CG}. The errors show that the presented \mbox{MRC-CG} provides state-of-the-art results for feature selection and can effectively select the most relevant features.

\section{Conclusion}
In this paper, we presented a learning algorithm for the recently proposed \acfp{MRC} that is efficient in \mbox{high-dimensional} settings. The algorithm utilizes a greedy feature selection approach that iteratively removes and selects features achieving a fast decrease in worst-case error probability while using a small number of features. We prove theoretically that the proposed iterative algorithm obtains a sequence of \acp{MRC} with decreasing worst-case error probabilities that converge to the solution obtained using all the features. The numerical results asses the efficiency of the presented algorithm and compare it with the state-of-the-art using 13 high-dimensional datasets with a large number of features. The results show that the presented algorithm converges to the solution obtained using all the features in a few iterations, and provides a significant efficiency increase, especially in cases with a large number of features. In addition, the algorithm provides the worst-case probability error that is particularly useful in high-dimensional scenarios with limited number of samples, which can suffer from a high variability of assessments based on \mbox{cross-validation}. The results also show that the algorithm can provide state-of-the-art results for feature selection and can effectively select the most relevant features.

\begin{acknowledgements} 
Funding in direct support of this work has been provided by projects PID2019-105058GA-I00, CNS2022-135203, and CEX2021-001142-S funded by MCIN/AEI/10.13039/501100011033 and the European Union “NextGenerationEU”/PRTR, and by programmes ELKARTEK and BERC-2022-2025 funded by the Basque Government.

\end{acknowledgements}

\bibliography{bondugula_793}
\end{document}